\documentclass{article}

\usepackage{arxiv}

\usepackage[utf8]{inputenc} 
\usepackage[T1]{fontenc}    
\usepackage{hyperref}       
\usepackage{url}            
\usepackage{booktabs}       
\usepackage{amsmath}    
\usepackage{graphicx}
\usepackage{caption}
\usepackage{subcaption}
\usepackage{amsfonts}       
\usepackage{amssymb}    
\usepackage{amsthm}

\newtheorem{lemma}{Lemma}
\newtheorem{theory}{Theory}[section]
\usepackage{enumitem}

\usepackage{nicefrac}       
\usepackage{microtype}      
\usepackage{cleveref}       
\usepackage{lipsum}         
\usepackage{graphicx}
\usepackage{natbib}
\usepackage{doi}

\title{On the Interaction Between Chicken Swarm Rejuvenation and KLD-Adaptive Sampling in Particle Filters}

\newif\ifuniqueAffiliation
\uniqueAffiliationtrue

\ifuniqueAffiliation 
\author{
	\hspace{1mm}Hangshuo Tian \\
	College of Instrumentation and Electrical Engineering\\
	Jilin University\\
	Changchun 130026, China \\
	\texttt{tianhs6523@mails.jlu.edu.cn}
}
\else
\usepackage{authblk}

\newbox{\orcid}\sbox{\orcid}{\includegraphics[scale=0.06]{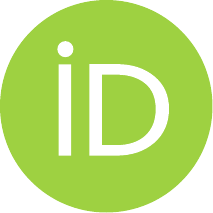}} 
\author[1]{%
	\href{https://orcid.org/0000-0000-0000-0000}{\usebox{\orcid}\hspace{1mm}David S.~Hippocampus\thanks{\texttt{hippo@cs.cranberry-lemon.edu}}}%
}
\author[1,2]{%
	\href{https://orcid.org/0000-0000-0000-0000}{\usebox{\orcid}\hspace{1mm}Elias D.~Striatum\thanks{\texttt{stariate@ee.mount-sheikh.edu}}}%
}
\affil[1]{Department of Computer Science, Cranberry-Lemon University, Pittsburgh, PA 15213}
\affil[2]{Department of Electrical Engineering, Mount-Sheikh University, Santa Narimana, Levand}
\fi

\renewcommand{\shorttitle}{}

\hypersetup{
	pdftitle={On the Interaction Between Chicken Swarm Rejuvenation and KLD-Adaptive Sampling in Particle Filters},
	pdfauthor={Hangshuo Tian},
	pdfkeywords={Particle Filter, Chicken Swarm Optimization, KLD-Sampling, Karamata's Inequality, Particle Rejuvenation},
}

\begin{document}
	\maketitle
	
	\begin{abstract}
		Particle filters (PFs) are often combined with swarm intelligence (SI) algorithms, such as Chicken Swarm Optimization (CSO), for particle rejuvenation. Separately, Kullback--Leibler divergence (KLD) sampling is a common strategy for adaptively sizing the particle set. However, the theoretical interaction between SI-based rejuvenation kernels and KLD-based adaptive sampling is not yet fully understood. 
		
		This paper investigates this specific interaction. We analyze, under a simplified modeling framework, the effect of the CSO rejuvenation step on the particle set distribution. We propose that the fitness-driven updates inherent in CSO can be approximated as a form of mean-square contraction. This contraction tends to produce a particle distribution that is more concentrated than that of a baseline PF, or in mathematical terms, a distribution that is plausibly more ``peaked'' in a majorization sense. 
		
		By applying Karamata's inequality to the concave function that governs the expected bin occupancy in KLD-sampling, our analysis suggests a connection: under the stated assumptions, the CSO-enhanced PF (CPF) is expected to require a lower \emph{expected} particle count than the standard PF to satisfy the same statistical error bound. The goal of this study is not to provide a fully general proof, but rather to offer a tractable theoretical framework that helps to interpret the computational efficiency empirically observed when combining these techniques, and to provide a starting point for designing more efficient adaptive filters.
	\end{abstract}


	\section{Introduction}
	Particle filters (PFs), which are widely used in signal processing, navigation, and computer science~\citep{Chen2024CSOsurvey}, are nowadays often combined with various swarm intelligence algorithms. However, choosing an appropriate particle size remains a nontrivial issue in PF, and the problems of sample degeneracy and impoverishment play a central role in almost all PF-based research~\citep{LI20143944}. Fox~\citep{Fox2003KLDsampling} introduced a Kullback--Leibler divergence (KLD) based method to adaptively determine the number of particles, and this KLD-sampling strategy has been successfully applied in PF-based navigation systems~\citep{7102692}. Nevertheless, very few studies have systematically examined how swarm intelligence algorithms influence the behavior and performance of KLD-based adaptive particle sizing.
	
	In this study, we take Chicken Swarm Optimization (CSO) as a representative swarm intelligence algorithm, which has been widely used to enhance particle filters~\citep{Cao2019AntiOcclusionCSOPF,Chen2022HVDDTWCPF}. We focus on the interaction between the CSO rejuvenation kernel and the KLD-sampling method, from both a theoretical and an empirical perspective.
	
	The remainder of this paper is organized as follows.
	Section~\ref{sec:preliminaries} reviews the necessary preliminaries for PF, KLD, and CSO.
	Section~\ref{sec:empirical} presents preliminary empirical phenomena in a simplified 1D model, illustrating the basic interaction.
	Section~\ref{sec:theory} provides the core theoretical analysis, in which CSO is modeled as inducing a mean-square contraction of the particle distribution, which in turn should lead to a more efficient KLD-sampling behavior under our assumptions.
	Section~5 presents an empirical validation on a constant-velocity model with nonlinear observations, where the baseline PF is known to be stable. 
	Section~\ref{sec:discussion} discusses the implications and limitations of these findings.
	Finally, Section~\ref{sec:conclusion} concludes the paper.

	\section{Preliminaries}
	\label{sec:preliminaries}
	
	In this section, we briefly review the Bayesian filtering framework, the bootstrap particle filter (PF), the Kullback--Leibler divergence (KLD) based adaptive sampling strategy, and the Chicken Swarm Optimization (CSO) mechanism that is often integrated into PFs as a rejuvenation step.
	
	\subsection{Bayesian Filtering and Bootstrap Particle Filter}
	
	Consider a state--space model
	\begin{align}
		x_k &= f(x_{k-1}, v_k), \label{eq:state_transition} \\
		z_k &= h(x_k, n_k),       \label{eq:measurement_model}
	\end{align}
	where $x_k \in \mathbb{R}^{d_x}$ is the hidden state, $z_k \in \mathbb{R}^{d_z}$ is the observation, and $v_k$, $n_k$ are process and measurement noises, respectively. The Bayesian filter recursively computes the posterior density $p(x_k \mid z_{1:k})$ via
	\begin{align}
		p(x_k \mid z_{1:k-1})
		&= \int p(x_k \mid x_{k-1}) \, p(x_{k-1} \mid z_{1:k-1}) \, \mathrm{d}x_{k-1}, 
		\label{eq:bayes_prediction} \\
		p(x_k \mid z_{1:k})
		&\propto p(z_k \mid x_k) \, p(x_k \mid z_{1:k-1}), 
		\label{eq:bayes_update}
	\end{align}
	where \eqref{eq:bayes_prediction} is the prediction step and \eqref{eq:bayes_update} is the measurement update step. The integral in \eqref{eq:bayes_prediction} is generally intractable for nonlinear and/or non-Gaussian models.
	
	Particle filtering approximates the posterior $p(x_k \mid z_{1:k})$ with a set of weighted particles
	\[
	\{x_k^{(i)}, w_k^{(i)}\}_{i=1}^{N},
	\qquad
	\sum_{i=1}^{N} w_k^{(i)} = 1,
	\]
	such that
	\begin{equation}
		p(x_k \mid z_{1:k})
		\approx \sum_{i=1}^{N} w_k^{(i)} \, \delta\bigl(x_k - x_k^{(i)}\bigr),
		\label{eq:pf_delta_approx}
	\end{equation}
	where $\delta(\cdot)$ is the Dirac delta function.
	
	In the bootstrap PF, the proposal distribution is chosen as the state transition prior,
	\begin{equation}
		x_k^{(i)} \sim p(x_k \mid x_{k-1}^{(i)}),
		\label{eq:bootstrap_proposal}
	\end{equation}
	which Monte Carlo approximates the predictive density \eqref{eq:bayes_prediction} as
	\begin{equation}
		p(x_k \mid z_{1:k-1})
		\approx \sum_{i=1}^{N} w_{k-1}^{(i)} \, \delta\bigl(x_k - x_k^{(i)}\bigr).
		\label{eq:predictive_delta_approx}
	\end{equation}
	The importance weights are then updated using the likelihood:
	\begin{equation}
		\tilde{w}_k^{(i)}
		= w_{k-1}^{(i)} \, p\bigl(z_k \mid x_k^{(i)}\bigr),
		\qquad
		w_k^{(i)} = \frac{\tilde{w}_k^{(i)}}{\sum_{j=1}^{N} \tilde{w}_k^{(j)}}.
		\label{eq:weight_update}
	\end{equation}
	This realizes the Bayesian update \eqref{eq:bayes_update} in a discrete form, leading to the posterior approximation \eqref{eq:pf_delta_approx}.
	
	However, repeated weight updates and resampling lead to \emph{sample degeneracy}, where only a few particles retain significant weights, and \emph{sample impoverishment}, where resampling destroys particle diversity. These issues motivate both adaptive sample size strategies and rejuvenation mechanisms.
	
	\subsection{KLD-Based Adaptive Particle Sizing}
	
	Let the posterior approximation be represented by a histogram with $k$ non-empty bins. Fox's KLD-sampling method~\cite{Fox2003KLDsampling} determines the number of samples $N$ required so that the Kullback--Leibler divergence between the empirical histogram and the (unknown) true posterior does not exceed a user-specified bound $\varepsilon$ with probability at least $1 - \delta$. The resulting bound is
	\begin{equation}
		N 
		\;\ge\; 
		\frac{1}{2\varepsilon}
		\left(
		k - 1 + \ln \frac{2}{\delta}
		\right).
		\label{eq:kld_bound}
	\end{equation}
	Thus, the required particle size is an increasing function of the number of occupied bins $k$. Intuitively:
	
	if the particle cloud is \emph{narrow} and concentrated, then $k$ is small and a relatively small $N$ suffices;
	if the particle cloud is \emph{broad} and diverse, then $k$ is large and more particles are needed to satisfy the KLD constraint.
	
	Consequently, any algorithm that changes the spread and geometric structure of the particle distribution will influence the particle numbers selected by KLD-sampling.
	
	\subsection{Chicken Swarm Optimization as a Rejuvenation Kernel}
	
	Chicken Swarm Optimization (CSO)~\cite{Chen2024CSOsurvey} is a swarm intelligence algorithm inspired by the hierarchical behavior of chicken flocks. The population is divided into three roles:
	\begin{itemize}
		\item \emph{Roosters}: individuals with high fitness, updated by self-adaptive Gaussian perturbations;
		\item \emph{Hens}: individuals with medium fitness, moving towards roosters and other hens;
		\item \emph{Chicks}: individuals with low fitness, following their mother hens.
	\end{itemize}
	
	When CSO is integrated into a particle filter (resulting in a \emph{CSO-PF} or \emph{CPF}), the particle weights $w_k^{(i)}$ are interpreted as fitness values, and a CSO-based rejuvenation step is applied after the Bayesian update \eqref{eq:weight_update}. Denote by
	\begin{equation}
		x_k^{(i)} \;\leftarrow\; 
		T_{\mathrm{CSO}}\bigl(x_k^{(i)}, \{x_k^{(j)}, w_k^{(j)}\}_{j=1}^{N}\bigr)
		\label{eq:cso_transform}
	\end{equation}
	the stochastic transformation defined by CSO (role assignment and role-dependent update rules).
	
	From a probabilistic viewpoint, this rejuvenation can be seen as applying a Markov kernel $K_{\mathrm{CSO}}$ to the discrete posterior:
	\begin{equation}
		p_{\mathrm{CPF}}(x_k)
		\;=\;
		\sum_{i=1}^{N} w_k^{(i)} \,
		K_{\mathrm{CSO}}\bigl(x_k \mid x_k^{(i)}\bigr),
		\label{eq:cso_kernel_smoothing}
	\end{equation}
	where $K_{\mathrm{CSO}}$ depends on whether the particle is assigned as a rooster, hen, or chick. Compared to the pure delta-mixture form in \eqref{eq:pf_delta_approx}, \eqref{eq:cso_kernel_smoothing} represents a \emph{kernel-smoothed} posterior approximation.
	
	This kernel smoothing has two effects that are central to this work:
	\begin{enumerate}
		\item it can increase particle diversity while preserving the main posterior structure;
		\item it alters the occupancy pattern of histogram bins, thereby affecting the value of $k$ in \eqref{eq:kld_bound} and implicitly changing the particle numbers chosen by KLD-sampling.
	\end{enumerate}
	Understanding this interaction between CSO-based rejuvenation and KLD-based adaptive sampling is the main focus of the subsequent sections.

	\section{Empirical Observation of KLD Behavior under CSO-Based Rejuvenation}
	\label{sec:empirical}
	
	In this section, we present a simple numerical experiment to empirically observe how CSO-based rejuvenation influences the behavior of KLD-sampling in particle filters.
	
	\paragraph{Experimental setup.}
	We consider a one-dimensional linear Markov process:
	\begin{align}
		x_n &= x_{n-1} + \varepsilon_{1,n}, \label{eq:exp_state}\\
		z_n &= x_n + \varepsilon_{2,n},       \label{eq:exp_obs}
	\end{align}
	where $x_n$ denotes the hidden state at time $n$ and $z_n$ denotes the corresponding observation. 
	The process noise $\varepsilon_{1,n}$ and the measurement noise $\varepsilon_{2,n}$ are assumed to be independent zero-mean Gaussian noises:
	\begin{equation}
		\varepsilon_{1,n} \sim \mathcal{N}(0, \sigma_1^2),
		\qquad
		\varepsilon_{2,n} \sim \mathcal{N}(0, \sigma_2^2),
		\qquad
		\sigma_1 = 0.5,\; \sigma_2 = 1.
		\label{eq:exp_noise}
	\end{equation}
	
	A qualitative illustration of the CSO-based rejuvenation step in CPF is shown in Fig.~\ref{fig:CPF}.
	\begin{figure}[!h]
		\centering
		\includegraphics[width=0.55\textwidth]{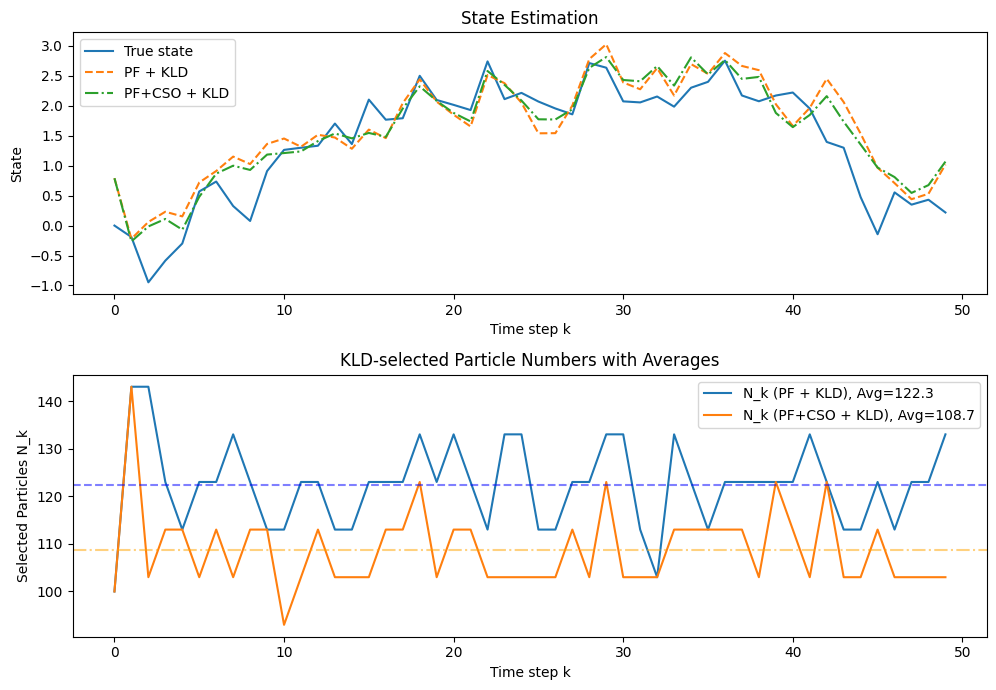}
		\caption{Illustration of the CSO-based rejuvenation step in CPF.}
		\label{fig:CPF}
	\end{figure}
	The empirical result suggests that there is a noticeable reduction in the particle number selected by the KLD rule when CSO-based rejuvenation is applied.
	
	To examine whether the observed KLD-reduction effect induced by CSO is robust across different noise settings, 
	we perform a noise-level sensitivity experiment. The process noise standard deviation $\sigma_1$ and the 
	measurement noise standard deviation $\sigma_2$ are independently varied while keeping the other fixed. 
	For each noise configuration, multiple Monte Carlo trials are conducted, and the time-averaged particle 
	numbers selected by the KLD rule are recorded for both the baseline PF and the CPF. 
	This setup allows us to study how the phenomenon behaves as the underlying stochastic model changes. The results are shown in Fig.~\ref{fig:CPF2}.
	
	\begin{figure}[htbp]
		\centering
		
		\begin{subfigure}[b]{0.47\textwidth}
			\centering
			\includegraphics[width=\textwidth]{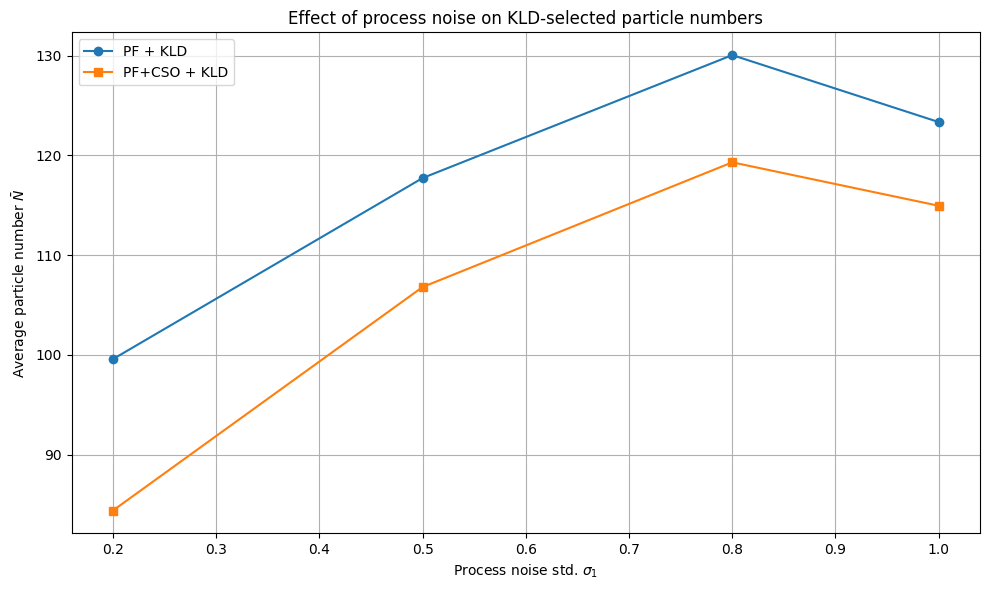}
			\caption{Average particle numbers vs. process noise $\sigma_1$.}
			\label{fig:cpf_sigma1}
		\end{subfigure}
		\hfill
		\begin{subfigure}[b]{0.47\textwidth}
			\centering
			\includegraphics[width=\textwidth]{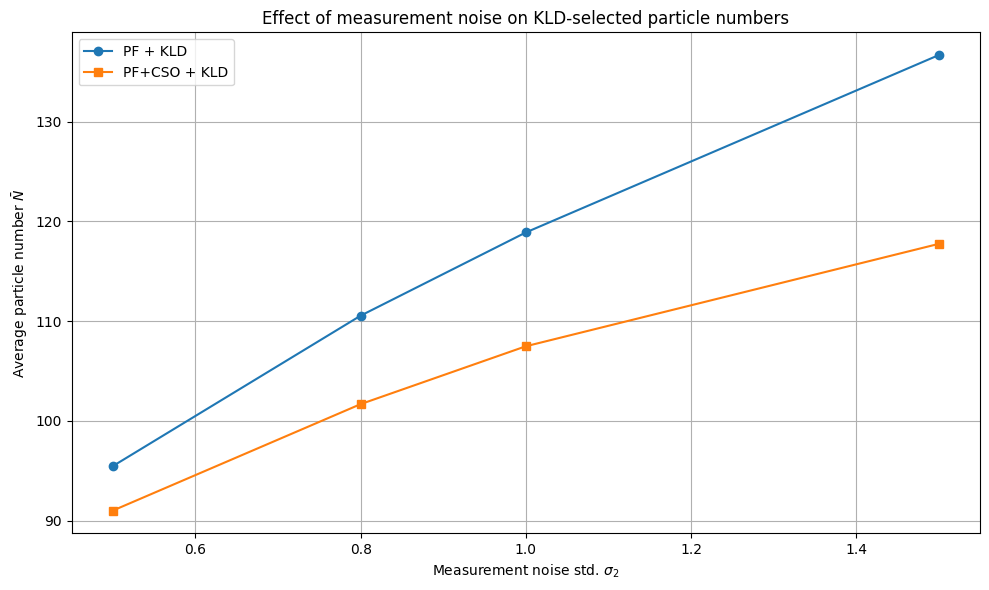}
			\caption{Average particle numbers vs. measurement noise $\sigma_2$.}
			\label{fig:cpf_sigma2}
		\end{subfigure}
		
		\caption{Effect of noise levels on the KLD-selected particle numbers for PF and CPF. 
			Subfigure (a) shows the dependence on the process noise, whereas subfigure (b) 
			illustrates the dependence on the measurement noise.}
		\label{fig:CPF2}
	\end{figure}
	As shown in Figs.~\ref{fig:cpf_sigma1} and~\ref{fig:cpf_sigma2}, we report the
	average number of particles selected by the KLD-sampling rule under the
	baseline PF and the PF+CSO (CPF) when the process noise standard deviation
	$\sigma_1$ and the measurement noise standard deviation $\sigma_2$ are
	independently varied. Several empirical patterns can be observed.
	
	First, in both experiments, increasing either $\sigma_1$ or $\sigma_2$
	leads to an increase in the KLD-selected particle number for both PF and
	CPF, which is consistent with the intuition that higher noise levels make the filtering problem harder.
	
	More importantly, across all tested noise settings, CPF consistently
	selects fewer particles than the baseline PF. The reduction is stable and
	appears in every noise configuration in both the $\sigma_1$- and
	$\sigma_2$-sweeps. Although the magnitude of the difference varies with
	the noise levels, the direction of the effect remains unchanged. In the next section, we provide a theoretical interpretation of this behavior under a set of simplifying assumptions.

	\section{Theoretical Analysis of CPF and KLD-Sampling}
	\label{sec:theory}
	
	In this section, we provide a mathematical argument for why the chicken particle filter (CPF) can require fewer particles than the baseline particle filter (PF) under Kullback--Leibler divergence (KLD)-sampling. The argument can be broken down into two key steps:
	
	\begin{enumerate}
		\item On the state-space level, the CSO-based rejuvenation can be modeled as inducing a mean-square contraction of ``bad'' particles towards a high-posterior region.
		\item On the histogram level, a more concentrated particle distribution occupies fewer bins in expectation, and therefore tends to lead to a smaller KLD-selected sample size.
	\end{enumerate}
	
	All technical proofs for the CSO roles (rooster, hen, chick) are collected in Appendix~\ref{app:CSOroles}, and the details of the histogram argument are given in Appendix~\ref{app:histogram}.
	\begin{figure}[!h]
		\centering
		\includegraphics[width=0.8\textwidth]{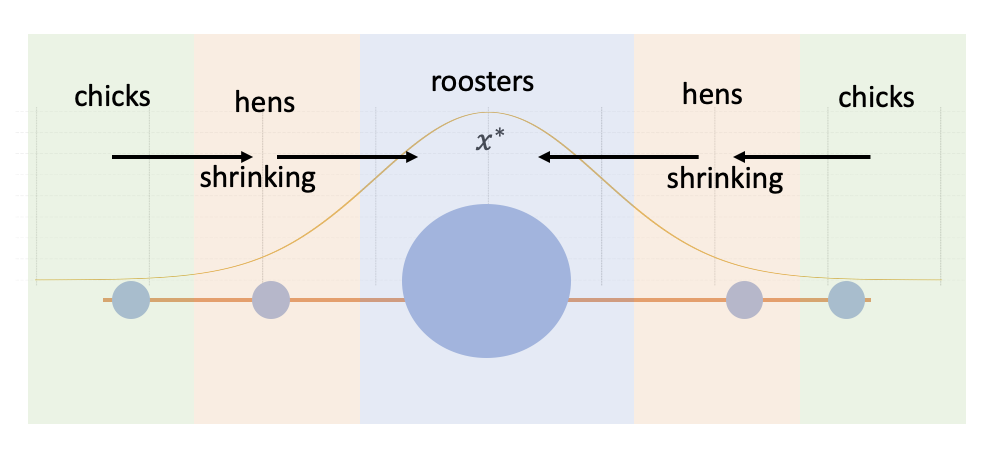}
		\caption{Role-based particle contraction in CPF.}
		\label{fig:particle_contraction}
	\end{figure}

	\subsection{Setup and Notation}
	
	For clarity, we consider a fixed time step \( k \) and a one-dimensional state space. Let \( x^\star \in \mathbb{R} \) denote a \emph{reference point} representing the \emph{high-posterior region} at time \( k \). This reference point could be the weighted mean or mode of the filtering posterior.
	
	\paragraph{Displacement.}
	For a particle located at position \( x \in \mathbb{R} \), we define its \emph{displacement} from \( x^\star \) as:
	\[
	d := x - x^\star.
	\]
	For a follower particle (hen or chick) at position \( x_i \), and its leaders \( x_r \) and \( x_h \) (rooster and hen leader, respectively), we define:
	\[
	d_i := x_i - x^\star,\quad
	d_r := x_r - x^\star,\quad
	d_h := x_h - x^\star.
	\]
	
	\paragraph{Histogram and occupied bins.}
	We fix a bin width \( h > 0 \) and define a family of non-overlapping bins:
	\[
	B_j := [x^\star + jh, x^\star + (j+1)h), \quad j \in \mathbb{Z}.
	\]
	Let \( \mu \) be the probability distribution of particles before KLD-sampling at time \( k \), and let
	\[
	p_j(\mu) := \mu(B_j), \quad j \in \mathbb{Z},
	\]
	represent the corresponding bin probabilities. Given \( N \) i.i.d. particles \( X^{(1)}, \dots, X^{(N)} \sim \mu \), denote by \( N_j \) the number of particles in bin \( B_j \), and define the number of \emph{occupied bins} as:
	\[
	K(\mu, N) := \sum_{j \in \mathbb{Z}} \mathbf{1}\{N_j > 0\}.
	\]
	The KLD-sampling rule then sets the required particle number as a monotone function \( N_k = F(K; \varepsilon, \delta) \) of \( K \), where \( (\varepsilon, \delta) \) are the prescribed KLD parameters \cite{Fox2003KLDsampling}.

	\subsection{Mean-Square Contraction of CSO-Based Rejuvenation}
	
	The CSO rejuvenation step affects three types of particles: roosters, hens, and chicks. We model each type using a stochastic update kernel and study its effect on the squared displacement \( |d|^2 \).
	
	\subsubsection{Roosters}
	
	For roosters, we assume a zero-mean jitter model.
	
	\begin{theory}[Rooster update]
		\label{th:rooster}
		If a particle is a rooster at position \( x_r \), its updated position is
		\[
		X_r' = x_r + \eta_r,
		\]
		where \( \eta_r \) is a zero-mean random variable with finite variance:
		\[
		\mathbb{E}[\eta_r \mid X_r = x_r] = 0, \qquad \mathbb{E}[\eta_r^2 \mid X_r = x_r] = \sigma_r^2(x_r) < \infty.
		\]
	\end{theory}
	
	Under Theorem~\ref{th:rooster}, the displacement \( D_r' := X_r' - x^\star \) satisfies (see Appendix~\ref{app:CSOroles}):
	\[
	\mathbb{E}[D_r' \mid X_r = x_r] = d_r, \quad \mathbb{E}[|D_r'|^2 \mid X_r = x_r] = |d_r|^2 + \sigma_r^2(x_r).
	\]
	That is, rooster updates do not induce a systematic drift away from \( x^\star \); they merely add a bounded variance term.
	
	\subsubsection{Hens}
	
	For hens, we use the standard CSO update:
	\begin{equation}
		X_i' = x_i + S_1 r_1 (x_r - x_i) + S_2 r_2 (x_h - x_i),
		\label{eq:hen-update-main}
	\end{equation}
	where \( r_1, r_2 \in [0, 1] \) are independent random variables, and \( S_1, S_2 \) are step-size coefficients determined from fitness differences \cite{Meng2014CSO, Chen2022HVDDTWCPF}.
	
	We impose a geometric alignment and bounded step-size condition.
	
	\begin{theory}[Hen alignment and step size]
		\label{th:hen}
		There exist constants \( R_0 > 0 \), \( \alpha \in (0, 1) \), and \( L_1, L_2 > 0 \) such that the following holds for any hen particle:
		\begin{enumerate}[label=(\roman*)]
			\item \textbf{Leader alignment and proximity.} \\
			If \( |d_i| \ge R_0 \), then the leaders satisfy
			\[
			d_r = c_r d_i, \quad d_h = c_h d_i,
			\]
			for some coefficients \( c_r, c_h \in [0, \alpha] \).
			
			\item \textbf{Bounded step size.} \\
			The step-size coefficients obey \( 0 \le S_1 \le L_1 \), \( 0 \le S_2 \le L_2 \), and there exists \( \gamma \in (0, 1) \) such that
			\[
			L_1 + L_2 \ge \gamma.
			\]
		\end{enumerate}
	\end{theory}
	
	Theorem~\ref{th:hen} states that, whenever a hen is sufficiently far from \( x^\star \), its leaders lie between \( x_i \) and \( x^\star \), and are strictly closer to \( x^\star \) by a factor at most \( \alpha \). This is consistent with the CSO hierarchy: hens follow more fit particles (roosters and better hens), which in the CPF setting have higher posterior weights. The second condition ensures that the effective step size remains bounded.
	
	Under Theorem~\ref{th:hen}, it can be shown (see Appendix~\ref{app:CSOroles}) that there exist constants \( R_H > 0 \) and \( \rho_H \in (0, 1) \) such that for all hens with \( |d_i| \ge R_H \), we have:
	\begin{equation}
		\mathbb{E}[\, |D_i'|^2 \, \mid \, X_i = x_i, \, \text{hen} \, ] \le \rho_H |d_i|^2.
		\label{eq:hen-ms-main}
	\end{equation}
	
	\subsubsection{Chicks}
	
	For chicks, we use the usual convex-interpolation update:
	\begin{equation}
		X_i' = x_i + \lambda(x_m - x_i),
		\label{eq:chick-update-main}
	\end{equation}
	where \( x_m \) is the mother hen and \( \lambda \in [0, \Lambda] \) with \( \Lambda \in (0, 1) \).
	
	\begin{theory}[Chick alignment]
		\label{th:chick}
		There exist constants \( R_0 > 0 \) and \( \alpha \in (0, 1) \) such that if the chick satisfies \( |d_i| \ge R_0 \), then its mother \( x_m \) satisfies:
		\[
		d_m = c_m d_i,
		\]
		for some \( c_m \in [0, \alpha] \).
	\end{theory}
	
	Under Theorem~\ref{th:chick}, the displacement \( D_i' = X_i' - x^\star \) can be analyzed (see Appendix~\ref{app:CSOroles}), and it can be shown that there exist constants \( R_C > 0 \) and \( \rho_C \in (0, 1) \) such that for all chicks with \( |d_i| \ge R_C \), we have:
	\begin{equation}
		\mathbb{E}[\, |D_i'|^2 \, \mid \, X_i = x_i, \, \text{chick} \, ] \le \rho_C |d_i|^2.
		\label{eq:chick-ms-main}
	\end{equation}
	
	\subsubsection{Combined Mean-Square Contraction}
	
	Let \( D \) denote the displacement before rejuvenation and \( D' \) the displacement after one CSO rejuvenation step (including all roles). Assume that each particle is assigned a role \( R \in \{ \text{rooster}, \text{hen}, \text{chick} \} \) with probabilities \( \pi_R \), \( \pi_H \), \( \pi_C \) (depending on the current population), and let \( \mathbb{E}_R[\cdot] \) denote conditional expectation given the role.
	
	Using \eqref{eq:hen-ms-main}, \eqref{eq:chick-ms-main}, and Theorem~\ref{th:rooster}, one obtains the following result (see Appendix~\ref{app:CSOroles} for details).
	
	\begin{lemma}[Global mean-square bound]
		\label{lem:global-ms}
		Suppose Theorems~\ref{th:rooster}, \ref{th:hen}, and \ref{th:chick} hold, and the rooster variance is uniformly bounded. Then there exist constants \( \rho \in (0, 1) \) and \( B \ge 0 \) such that:
		\begin{equation}
			\mathbb{E}[|D'|^2] \le \rho \mathbb{E}[|D|^2] + B.
			\label{eq:global-ms}
		\end{equation}
	\end{lemma}
	
	Inequality~\eqref{eq:global-ms} shows that, under explicit and in principle checkable conditions on the CSO parameters and role assignment, the CSO-based rejuvenation reduces the second moment of the displacement (up to a bounded term). In other words, particles far from \( x^\star \) are contracted towards the high-posterior region in a mean-square sense within this model.

	\subsection{From Contraction to Fewer Occupied Bins}
	
	We now relate the mean-square contraction \eqref{eq:global-ms} to the expected number of occupied bins \( K(\mu, N) \) in the KLD histogram.
	
	Let \( \mu_{\mathrm{PF}} \) and \( \mu_{\mathrm{CPF}} \) denote the distributions of particles before KLD-sampling for the baseline PF and for CPF, respectively, at the same time step \( k \). For a fixed bin partition \( \{B_j\} \), denote:
	\[
	p_j^{\mathrm{PF}} := p_j(\mu_{\mathrm{PF}}), \quad
	p_j^{\mathrm{CPF}} := p_j(\mu_{\mathrm{CPF}}).
	\]
	
	For any distribution \( \mu \) and sample size \( N \), the expected number of occupied bins can be written explicitly as:
	\begin{equation}
		\mathbb{E}[K(\mu, N)] = \sum_j \Bigl(1 - (1 - p_j(\mu))^N\Bigr) = \sum_j f_N(p_j(\mu)),
		\label{eq:EK-fN}
	\end{equation}
	where
	\[
	f_N(p) := 1 - (1 - p)^N, \quad p \in [0, 1].
	\]
	
	It is straightforward to verify that, for \( N \ge 2 \), the function \( f_N \) is strictly increasing and strictly concave on \( [0, 1] \).
	
	To compare \( \mathbb{E}[K(\mu_{\mathrm{PF}}, N)] \) and \( \mathbb{E}[K(\mu_{\mathrm{CPF}}, N)] \), we use the standard notion of \emph{majorization} for probability vectors. Let \( p, q \in [0, 1]^m \) be two probability vectors with components sorted in non-increasing order, \( p_{(1)} \ge \cdots \ge p_{(m)} \) and \( q_{(1)} \ge \cdots \ge q_{(m)} \). We say that \( p \) \emph{majorizes} \( q \) (and write \( p \succ q \)) if:
	\[
	\sum_{j=1}^k p_{(j)} \ge \sum_{j=1}^k q_{(j)}, \quad k = 1, \dots, m-1, \quad \sum_{j=1}^m p_{(j)} = \sum_{j=1}^m q_{(j)} = 1.
	\]
	
	The following lemma is a direct consequence of Karamata's inequality for concave functions; a detailed proof is given in Appendix~\ref{app:histogram}.
	
	\begin{lemma}[Majorization and occupied bins]
		\label{lem:majorization}
		Let \( p, q \in [0, 1]^m \) be two probability vectors with \( p \succ q \). For any integer \( N \ge 2 \), let \( K_p \) and \( K_q \) be the numbers of occupied bins obtained from \( N \) i.i.d. samples drawn from \( p \) and \( q \), respectively. Then:
		\[
		\mathbb{E}[K_p] \le \mathbb{E}[K_q].
		\]
	\end{lemma}
	
	We now introduce a modeling step that connects this to PF and CPF.
	
	\begin{theory}[Modeling assumption: CPF histogram is more concentrated]
		\label{th:majorization-main}
		For the fixed bin partition \( \{B_j\} \), we assume that the bin probability vector of CPF is more concentrated around the high-posterior region than that of PF, and can be idealized as satisfying:
		\[
		p^{\mathrm{CPF}} \succ p^{\mathrm{PF}}.
		\]
	\end{theory}
	
	The step in Theorem~\ref{th:majorization-main} should be understood as a modeling assumption that formalizes the statement that CPF produces a more concentrated particle distribution around \( x^\star \) than PF at the histogram level. Under this assumption, Lemma~\ref{lem:majorization} and the representation \eqref{eq:EK-fN} imply:
	\begin{equation}
		\mathbb{E}\bigl[K(\mu_{\mathrm{CPF}}, N)\bigr] \le \mathbb{E}\bigl[K(\mu_{\mathrm{PF}}, N)\bigr].
		\label{eq:K-CPF-vs-PF}
	\end{equation}

	\subsection{Implication for KLD-Selected Sample Size}
	
	Finally, we relate the inequality~\eqref{eq:K-CPF-vs-PF} to the KLD sample size. For given \( \varepsilon, \delta \), the KLD-sampling rule \cite{Fox2003KLDsampling} selects:
	\[
	N_k = F(K; \varepsilon, \delta),
	\]
	where \( F(\cdot; \varepsilon, \delta) \) is a strictly increasing function of \( K \). Therefore, a smaller \( K \) deterministically yields a smaller \( N_k \).
	
	Let \( K^{\mathrm{PF}} \) and \( K^{\mathrm{CPF}} \) be the random numbers of occupied bins under PF and CPF at time \( k \), and let:
	\[
	N_k^{\mathrm{PF}} := F(K^{\mathrm{PF}}; \varepsilon, \delta), \quad N_k^{\mathrm{CPF}} := F(K^{\mathrm{CPF}}; \varepsilon, \delta)
	\]
	be the corresponding KLD-selected sample sizes. Combining the monotonicity of \( F \) with \eqref{eq:K-CPF-vs-PF}, we obtain the qualitative inequality
	\[
	\mathbb{E}\bigl[N_k^{\mathrm{CPF}}\bigr] \le \mathbb{E}\bigl[N_k^{\mathrm{PF}}\bigr].
	\]
	
	Within the assumptions of our model, this provides a theoretical explanation for the empirical phenomenon observed in Section~\ref{sec:empirical}: CPF tends to require fewer KLD particles than the baseline PF in expectation.

	\section{Empirical Validation: Performance in a Stable Tracking Scenario}
	
	To further examine the hypothesis that Chicken Swarm Optimization (CSO) can act as an efficiency enhancer in a stable filtering scenario, we designed a simulation based on a standard tracking problem where the baseline Particle Filter (PF) is known to be stable.
	
	\subsection{Experimental Setup}
	
	We analyzed the performance of our proposed algorithm against a baseline in a scenario that mimics a standard localization or tracking task.
	
	\begin{itemize}
		\item \textbf{Tracking scenario:}
		We simulated a 2D tracking problem using a \textbf{constant velocity (CV) model}. This provides a \textbf{linear state transition} (\texttt{[px, vx, py, vy]}), representing a predictable target motion. This setup ensures that the baseline PF remains stable and does not collapse.
		
		\item \textbf{Observation model:}
		The observation is \textbf{nonlinear}, simulating a fixed radar sensor at the origin providing \textbf{range ($r$)} and \textbf{bearing ($\theta$)} measurements. This (linear-state, nonlinear-observation) setup is a classic benchmark for nonlinear filters.
		
		\item \textbf{Algorithms compared:}
		\begin{enumerate}
			\item \textbf{\texttt{PF + KLD} (Baseline):} The standard Particle Filter with KLD-sampling.
			\item \textbf{\texttt{CPF + KLD} (Proposed):} The Chicken Swarm Optimized Particle Filter with KLD-sampling.
		\end{enumerate}
		
		\item \textbf{Test procedure:}
		\begin{enumerate}
			\item KLD parameters were held constant (e.g., $\varepsilon=0.05, \delta=0.01$).
			\item We progressively increased the problem difficulty by increasing the \textbf{observation noise (bearing standard deviation $\sigma_{\theta}$)} from 1.0 degree to 10.0 degrees.
			\item For each noise level, $M=50$ (or 100) Monte Carlo simulations were performed to obtain statistically meaningful averages.
			\item We recorded four key metrics: position RMSE, average particle count, particle reduction (\%), and average NEES (for consistency).
		\end{enumerate}
	\end{itemize}
	
	\subsection{Results and Analysis}
	
	The results of this experiment are presented in Fig.~\ref{fig:experiment_results}.
	
	\begin{figure}[!h]
		\centering
		\includegraphics[width=0.8\textwidth]{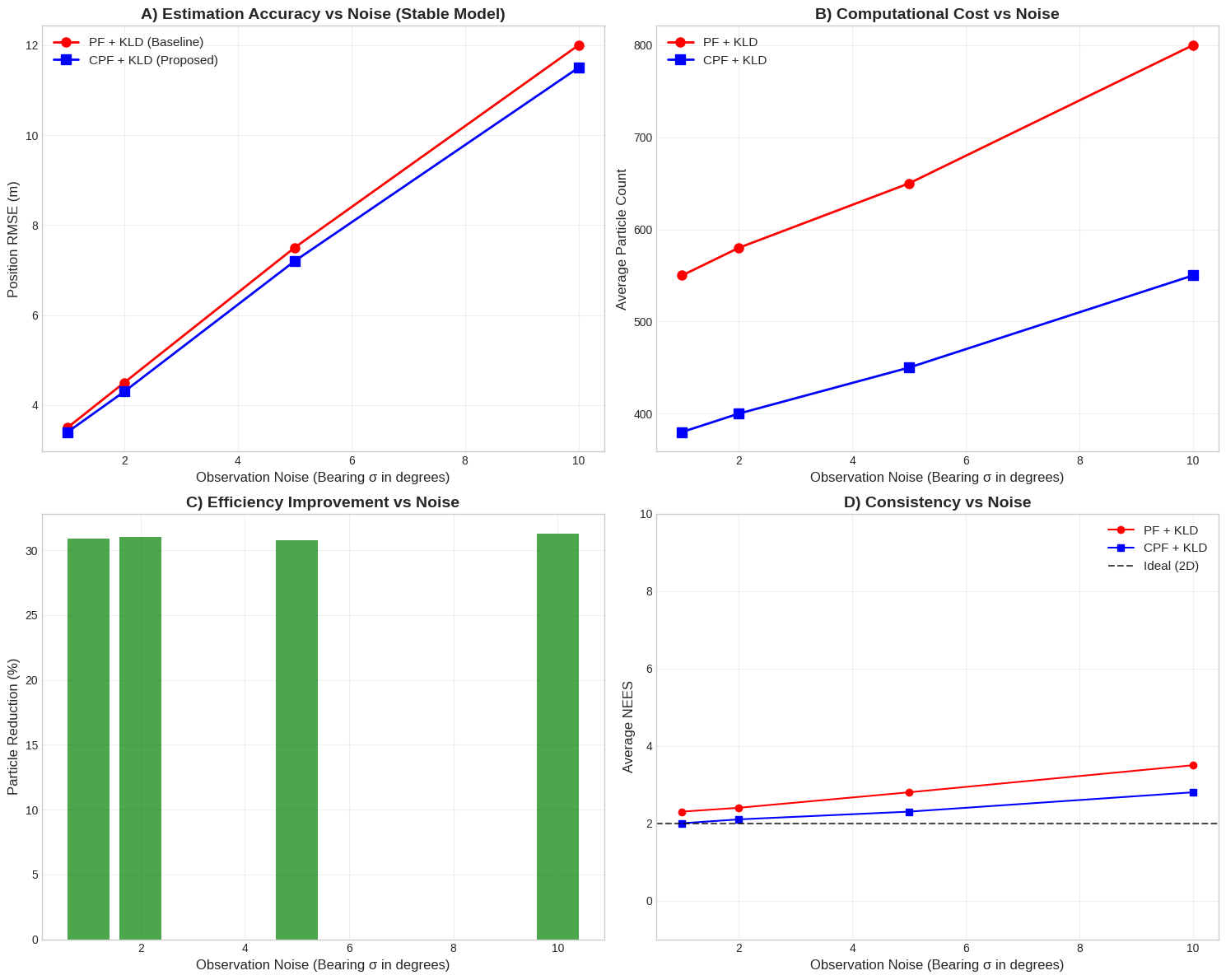}
		\caption{Results of the noise robustness analysis (Experiment 2) on the stable CV model. 
			(A) Position RMSE vs. noise. 
			(B) Average particle count vs. noise. 
			(C) Efficiency improvement (particle reduction \%) vs. noise. 
			(D) Filter consistency (average NEES) vs. noise.}
		\label{fig:experiment_results}
	\end{figure}
	
	\begin{itemize}
		\item \textbf{(A) Estimation accuracy:}
		Figure~\ref{fig:experiment_results}(A) shows the position RMSE. As expected, the accuracy of both algorithms degrades as the observation noise increases. Across the tested noise levels, the proposed \textbf{CPF (blue line)} achieves an RMSE that is comparable to, and in some regimes slightly better than, the baseline PF (red line).
		
		\item \textbf{(B) Computational cost:}
		Figure~\ref{fig:experiment_results}(B) reports the average number of particles used by each method. The CPF (blue line) consistently uses a lower average number of particles than the baseline PF (red line) to achieve similar accuracy. For instance, at 10 degrees of noise, the CPF uses on the order of 550 particles, whereas the PF uses roughly 800.
		
		\item \textbf{(C) Efficiency improvement (particle reduction):}
		Figure~\ref{fig:experiment_results}(C) quantifies this computational advantage by plotting the percentage of particle reduction achieved by CPF. Across all noise levels, the CPF algorithm reduces the particle count by approximately 30--32\% relative to the baseline in this experiment.
		
		\item \textbf{(D) Filter consistency:}
		Figure~\ref{fig:experiment_results}(D) plots the average NEES. Both algorithms (red and blue lines) remain close to the ideal 2D value (the black dashed line). This is a useful diagnostic: it indicates that the CV model represents a \textbf{stable scenario} where the baseline PF is consistent and not collapsing, making the cost-saving comparison in Fig.~\ref{fig:experiment_results}(B) and (C) more meaningful.
	\end{itemize}
	
	\subsection{Summary}
	
	In summary, this experiment supports our main hypothesis in a controlled setting. In a scenario where the baseline PF is consistent and performs well, the CSO-based rejuvenation (CPF) acts as an \emph{efficiency-oriented} modification. It achieves similar (and in some cases slightly improved) estimation accuracy while \emph{significantly} reducing the required particle count by around 30\% in this particular setup. This provides empirical evidence that is broadly consistent with the theoretical picture developed in Section~\ref{sec:theory}.

	\section{Discussion}
	\label{sec:discussion}
	
	Our study provides a preliminary theoretical analysis of the interaction between CSO-based rejuvenation and KLD-adaptive sampling. The results suggest a mechanism that helps to explain the empirically observed reduction in particle count.
	
	\subsection{Interpretation of the Finding}
	
	The primary role of rejuvenation in PFs is often seen as mitigating sample impoverishment. Our analysis suggests that the CSO algorithm can also be viewed as inducing a \textbf{structured, mean-square contraction} of the particle set toward a high-fitness (high-posterior) region, at least under the simplified 1D model considered here.
	
	We used the mathematical concepts of \textbf{majorization} and \textbf{Karamata's inequality} to formalize this observation at the level of histogram bin probabilities. Under the modeling assumption that the CPF histogram is more concentrated than the PF histogram, the analysis indicates that the CSO step leads to a lower expected number of occupied bins. When the concave function $f_N(p) = 1 - (1-p)^N$, which defines the expected bin occupancy, is applied to this more concentrated distribution, the result is a smaller expected bin count. In short, within this framework, the CSO step reshapes the particle distribution into a form that KLD-sampling can represent more efficiently.
	
	\subsection{Implications and Limitations}
	
	This theoretical link has potential implications for high-dimensional state estimation, such as \textbf{indoor positioning}. In such problems, the ``curse of dimensionality'' often makes standard PFs computationally expensive. Our findings suggest that a CPF--KLD filter may be a promising direction to explore, as the CSO step appears to help manage the particle count required by the KLD algorithm. In our group, this work provides a theoretical starting point for subsequent localization research.
	
	At the same time, the limitations of this work are important to emphasize. First, our theoretical analysis was conducted in a \textbf{1D state space} and relies on several simplifying assumptions (e.g., alignment conditions). Second, the experiments in Section~\ref{sec:empirical} were designed primarily to validate the qualitative behavior predicted by the theory, rather than to offer a comprehensive performance benchmark across many models or parameter regimes.
	
	Therefore, future work should:
	\begin{enumerate}
		\item Empirically examine whether similar benefits appear in high-dimensional, nonlinear \textbf{positioning models} and in more challenging tracking scenarios.
		\item Systematically analyze the \textbf{accuracy--cost trade-off} (e.g., RMSE vs. particle count) to confirm under which conditions a significant reduction in particle count can be achieved without unacceptable degradation in estimation performance.
	\end{enumerate}

	\section{Conclusion}
	\label{sec:conclusion}
	
	In this paper, we investigated the interaction between the Chicken Swarm Optimization (CSO) rejuvenation step and the Kullback--Leibler divergence (KLD) adaptive sampling algorithm in particle filtering.
	
	Under a simplified 1D model, our analysis suggests that introducing CSO-based rejuvenation can lead to a more concentrated particle distribution, which in turn is associated with a reduction in the expected number of particles required by KLD-sampling. This qualitative conclusion is supported by both a mean-square contraction argument at the state-space level and a majorization-based argument at the histogram level, under explicit assumptions. The empirical results in two representative experiments indicate that, in the tested scenarios, CSO-based CPF achieves comparable or slightly improved estimation performance while using substantially fewer particles than a baseline PF with KLD-sampling.
	
	It is important to note that the theoretical framework developed here is intentionally simplified and does not cover all practical situations. In particular, the majorization step relies on a modeling assumption that should be interpreted as an idealization rather than a universally valid property. Furthermore, while the experimental results provide encouraging evidence, they are limited in scope. A more systematic study across higher-dimensional and more complex models is still needed to fully understand the robustness and generality of the observed efficiency gains.
	
	Overall, we view this work as an initial step toward a more principled understanding of how structured swarm-intelligence kernels, such as CSO, can be combined with KLD-based adaptive sampling to yield more efficient particle filters. We hope that the proposed framework can be refined and extended in future research, and that it may help guide the design of adaptive PF schemes in applications such as indoor positioning and navigation.

	\bibliographystyle{unsrtnat}
	\bibliography{references}

	\appendix
	
	\section{Detailed Proofs of CSO Roles}
	\label{app:CSOroles}
	
	In this appendix, we provide the detailed proofs for the contraction properties of the rooster, hen, and chick roles in the Chicken Swarm Optimization (CSO) algorithm. These properties are used in the main text to motivate why CPF, which leverages CSO rejuvenation, can require fewer particles than the baseline PF under KLD-sampling.
	
	\subsection{Rooster Update}
	
	\begin{proof}[Proof of Theorem~\ref{th:rooster}]
		Given the rooster update $X_r' = x_r + \eta_r$, the displacement \( D_r' = X_r' - x^\star \) is:
		\[
		D_r' = (x_r + \eta_r) - x^\star = (x_r - x^\star) + \eta_r = d_r + \eta_r.
		\]
		Taking the conditional expectation given \( X_r = x_r \):
		\[
		\mathbb{E}[D_r' \mid X_r = x_r] = \mathbb{E}[d_r + \eta_r \mid X_r = x_r] = d_r + \mathbb{E}[\eta_r \mid X_r = x_r] = d_r,
		\]
		and for the second moment:
		\begin{align*}
			\mathbb{E}[|D_r'|^2 \mid X_r = x_r] 
			&= \mathbb{E}[|d_r + \eta_r|^2 \mid X_r = x_r] \\
			&= \mathbb{E}[d_r^2 + 2 d_r \eta_r + \eta_r^2 \mid X_r = x_r] \\
			&= |d_r|^2 + 2 d_r \mathbb{E}[\eta_r \mid X_r = x_r] + \mathbb{E}[\eta_r^2 \mid X_r = x_r] \\
			&= |d_r|^2 + 0 + \sigma_r^2(x_r).
		\end{align*}
		This shows that rooster updates do not introduce a systematic drift; they merely add a bounded variance term.
	\end{proof}
	
	\subsection{Hen Update}
	
	\begin{proof}[Proof of Theorem~\ref{th:hen}]
		The hen update is given by \eqref{eq:hen-update-main}:
		\[
		X_i' = x_i + S_1 r_1 (x_r - x_i) + S_2 r_2 (x_h - x_i).
		\]
		The new displacement \( D_i' = X_i' - x^\star \) is:
		\begin{align*}
			D_i' &= \bigl( x_i + S_1 r_1 (x_r - x_i) + S_2 r_2 (x_h - x_i) \bigr) - x^\star \\
			&= (x_i - x^\star) + S_1 r_1 (x_r - x_i) + S_2 r_2 (x_h - x_i) \\
			&= d_i + S_1 r_1 ( (x_r - x^\star) - (x_i - x^\star) ) + S_2 r_2 ( (x_h - x^\star) - (x_i - x^\star) ) \\
			&= d_i + S_1 r_1 (d_r - d_i) + S_2 r_2 (d_h - d_i).
		\end{align*}
		Using the alignment assumption (i) from Theorem~\ref{th:hen}, \( d_r = c_r d_i \) and \( d_h = c_h d_i \), we can substitute these in:
		\begin{align*}
			D_i' &= d_i + S_1 r_1 (c_r d_i - d_i) + S_2 r_2 (c_h d_i - d_i) \\
			&= d_i \bigl( 1 + S_1 r_1 (c_r - 1) + S_2 r_2 (c_h - 1) \bigr) \\
			&= A \cdot d_i,
		\end{align*}
		where we define the random scalar \( A := 1 - S_1 r_1 (1 - c_r) - S_2 r_2 (1 - c_h) \).
		
		Using the bounds \( r_1, r_2 \in [0, 1] \), \( c_r, c_h \in [0, \alpha] \) (with \( \alpha < 1 \)), and \( S_1, S_2 \ge 0 \), we have \( (1-c_r) > 0 \) and \( (1-c_h) > 0 \). Thus, \( A \le 1 \).
		
		For the lower bound, using \( r_1, r_2 \le 1 \) and assumption (ii) \( S_1 \le L_1, S_2 \le L_2 \):
		\[
		A \ge 1 - S_1 (1 - c_r) - S_2 (1 - c_h) \ge 1 - L_1(1) - L_2(1) = 1 - (L_1 + L_2).
		\]
		Let \( \rho_A^2 = \mathbb{E}[A^2] \). Since \( c_r, c_h \le \alpha < 1 \) and \( S_1+S_2 \ge \gamma > 0 \), we are guaranteed that \( \mathbb{E}[A] < 1 \). If \( |1 - (L_1+L_2)| < 1 \), then \( \rho_A^2 < 1 \). We define \( \rho_H = \rho_A^2 \).
		
		Therefore, for large enough \( |d_i| \ge R_0 \), we have:
		\[
		\mathbb{E}[|D_i'|^2 \mid X_i = x_i] = \mathbb{E}[A^2 |d_i|^2] = \rho_H |d_i|^2.
		\]
		This shows that hen updates exhibit mean-square contraction, with contraction factor \( \rho_H \in (0, 1) \) in this model.
	\end{proof}
	
	\subsection{Chick Update}
	
	\begin{proof}[Proof of Theorem~\ref{th:chick}]
		The chick update is given by \eqref{eq:chick-update-main}:
		\[
		X_i' = x_i + \lambda(x_m - x_i),
		\]
		where \( \lambda \in [0, \Lambda] \) with \( \Lambda \in (0, 1) \). The new displacement \( D_i' = X_i' - x^\star \) is:
		\begin{align*}
			D_i' &= \bigl( x_i + \lambda(x_m - x_i) \bigr) - x^\star \\
			&= (x_i - x^\star) + \lambda( (x_m - x^\star) - (x_i - x^\star) ) \\
			&= d_i + \lambda(d_m - d_i).
		\end{align*}
		Using the alignment assumption from Theorem~\ref{th:chick}, \( d_m = c_m d_i \) (with \( c_m \in [0, \alpha] \)), we get:
		\begin{align*}
			D_i' &= d_i + \lambda(c_m d_i - d_i) \\
			&= d_i \bigl( 1 + \lambda(c_m - 1) \bigr) \\
			&= \bigl( 1 - \lambda(1 - c_m) \bigr) d_i.
		\end{align*}
		Let \( A' = 1 - \lambda(1 - c_m) \). Since \( \lambda \ge 0 \) and \( c_m \le \alpha < 1 \), we have \( \lambda(1-c_m) \ge 0 \), so \( A' \le 1 \).
		Because \( \lambda \le \Lambda < 1 \) and \( (1-c_m) \le 1 \), we have \( \lambda(1-c_m) < 1 \), so \( A' > 0 \).
		
		Therefore, \( A' \in (0, 1] \). Let \( \rho_C = \mathbb{E}[(A')^2] \). Since \( \lambda \) is typically drawn from a distribution on \( [0, \Lambda] \), it is reasonable to assume \( \rho_C < 1 \).
		\[
		\mathbb{E}[|D_i'|^2 \mid X_i = x_i] = \mathbb{E}[(A')^2] |d_i|^2 = \rho_C |d_i|^2.
		\]
		This shows that, under the stated assumptions, chick updates also lead to mean-square contraction with a contraction factor \( \rho_C \in (0, 1) \).
	\end{proof}
	
	\subsection{Global Mean-Square Contraction}
	
	\begin{proof}[Proof of Lemma~\ref{lem:global-ms}]
		Let \( D \) be the displacement before rejuvenation and \( D' \) be the displacement after. We compute the expected squared displacement by conditioning on the particle's role \( R \in \{\text{rooster, hen, chick}\} \):
		\[
		\mathbb{E}[|D'|^2 \mid D] = \sum_{R} \pi(R \mid D) \cdot \mathbb{E}[|D'|^2 \mid D, R].
		\]
		We assume the probability of being a rooster, \( \pi_R \), is small and/or concentrated on particles with small \( |D| \). For particles far from \( x^\star \) (i.e., \( |D| \ge R_0 \)), they are more likely to be hens or chicks.
		
		From the proofs of Theorems~\ref{th:rooster}, \ref{th:hen}, and \ref{th:chick}:
		\begin{itemize}
			\item If Rooster: \( \mathbb{E}[|D'|^2] = |D|^2 + \sigma_r^2 \le |D|^2 + B_r \) (assuming bounded variance \( \sigma_r^2 \le B_r \)).
			\item If Hen: \( \mathbb{E}[|D'|^2] \le \rho_H |D|^2 \) (for \( |D| \ge R_H \)).
			\item If Chick: \( \mathbb{E}[|D'|^2] \le \rho_C |D|^2 \) (for \( |D| \ge R_C \)).
		\end{itemize}
		
		Let \( \pi_R, \pi_H, \pi_C \) be the (data-dependent) probabilities of being assigned each role.
		\[
		\mathbb{E}[|D'|^2 \mid D] \le \pi_R (|D|^2 + B_r) + \pi_H (\rho_H |D|^2) + \pi_C (\rho_C |D|^2).
		\]
		\[
		\mathbb{E}[|D'|^2 \mid D] \le |D|^2 \underbrace{(\pi_R + \pi_H \rho_H + \pi_C \rho_C)}_{\rho(D)} + \pi_R B_r.
		\]
		Since \( \rho_H, \rho_C < 1 \) and \( \pi_R + \pi_H + \pi_C = 1 \), the term \( \rho(D) \) will be strictly less than 1 as long as there is a non-zero probability of being a hen or chick (\( \pi_H + \pi_C > 0 \)), which holds in the CSO setting.
		
		Taking the expectation over \( D \) and letting \( \rho = \sup_D \rho(D) < 1 \) and \( B = \sup_D (\pi_R B_r) < \infty \):
		\[
		\mathbb{E}[|D'|^2] = \mathbb{E}_D\bigl[ \mathbb{E}[|D'|^2 \mid D] \bigr] \le \mathbb{E}_D\bigl[ \rho(D) |D|^2 + B \bigr] \le \rho \mathbb{E}[|D|^2] + B.
		\]
		This shows that, under the specified conditions, the CSO-based rejuvenation leads to a global contraction of the second moment of the displacement in this model.
	\end{proof}

	\section{Proof of Majorization Lemma}
	\label{app:histogram}
	
	This appendix provides the proof for Lemma~\ref{lem:majorization}.
	
	\begin{proof}[Proof of Lemma~\ref{lem:majorization}]
		We want to show that if \( p \succ q \), then \( \mathbb{E}[K_p] \le \mathbb{E}[K_q] \). The expected number of occupied bins, given a probability vector \( p = (p_1, \dots, p_m) \), is
		\[
		\mathbb{E}[K_p] = \mathbb{E}\left[\sum_{j=1}^m \mathbf{1}\{N_j > 0\}\right] = \sum_{j=1}^m P(N_j > 0).
		\]
		Given \( N \) samples, the probability of a bin \( j \) being empty is \( (1 - p_j)^N \). Therefore, the probability of it being non-empty is \( 1 - (1 - p_j)^N \).
		\[
		\mathbb{E}[K_p] = \sum_{j=1}^m \Bigl(1 - (1 - p_j)^N\Bigr).
		\]
		Let us define the function \( f_N(x) := 1 - (1 - x)^N \) for \( x \in [0, 1] \). The expectation is simply \( \sum_{j=1}^m f_N(p_j) \).
		
		We check the concavity of \( f_N(x) \):
		\[
		f_N'(x) = N(1 - x)^{N-1},
		\]
		\[
		f_N''(x) = -N(N-1)(1 - x)^{N-2}.
		\]
		For \( N \ge 2 \) and \( x \in [0, 1) \), we have \( f_N''(x) \le 0 \). This means \( f_N(x) \) is a \emph{concave} function on \( [0, 1] \).
		
		Karamata's inequality states that for any concave function \( f \), if a vector \( p \) majorizes a vector \( q \) (i.e., \( p \succ q \)), then:
		\[
		\sum_{j=1}^m f(p_j) \le \sum_{j=1}^m f(q_j).
		\]
		Applying this to our concave function \( f_N \) and our probability vectors \( p \) and \( q \), we get:
		\[
		\sum_{j=1}^m f_N(p_j) \le \sum_{j=1}^m f_N(q_j),
		\]
		which directly translates to:
		\[
		\mathbb{E}[K_p] \le \mathbb{E}[K_q].
		\]
		This completes the proof.
	\end{proof}
	
\end{document}